\documentclass[%
]{mpi2015-cscpreprint}

\usepackage{graphicx, caption, subcaption}

\usepackage[hardware]{mymacros}
\def\totimes{\mathbin{\tilde\otimes}}
\usepackage{cancel}
\newtheorem{theorem}{Theorem}

\newenvironment{proof}{\textit{Proof:}}{\hfill$\square$}

\usepackage{todonotes}

\usepackage[normalem]{ulem}

\usepackage{float}
\usepackage{color}
\usepackage{array}
\usepackage{dsfont}
\usepackage{tikz}
\usetikzlibrary{arrows,decorations.pathmorphing,backgrounds,fit,positioning,shapes.symbols,chains}
\usepackage{wrapfig}
\renewcommand{\hat}[1]{\widehat{#1}}

\usepackage{spalign}
\usepackage{mathtools}

\newcommand{\newmethod}{\texttt{stable-OpInfc}}
\newcommand{\benchmark}{\texttt{OpInfc}}

\newcommand{\pathfig}{./figures}

\usepackage{soul}
\usepackage{cleveref}
\usepackage{subcaption}

\begin{document}

\title{Stability-Certified Learning of Control Systems with Quadratic Nonlinearities} 

\author[$\dagger$]{Igor Pontes Duff}
\affil[$\dagger$]{Max Planck Institute of Dynamics of Complex Technical Systems\authorcr
	\email{pontes@mpi-magdeburg.mpg.de}, \orcid{0000-0001-6433-6142}}

\author[$\ddagger$]{Pawan Goyal}
\affil[$\ddagger$]{Max Planck Institute of Dynamics of Complex Technical Systems\authorcr
	\email{goyalp@mpi-magdeburg.mpg.de}, \orcid{0000-0003-3072-7780}}

\author[$\dagger\dagger$]{Peter Benner}
\affil[$\dagger\dagger$]{Max Planck Institute of Dynamics of Complex Technical Systems, Otto-von-Guericke University Magdeburg \authorcr
	\email{benner@mpi-magdeburg.mpg.de}, \orcid{0000-0003-3362-4103}}

\abstract{This work primarily focuses on an operator inference methodology aimed at constructing low-dimensional dynamical models based on a priori hypotheses about their structure, often informed by established physics or expert insights. Stability is a fundamental attribute of dynamical systems, yet it is not always assured in models derived through inference. Our main objective is to develop a method that facilitates the inference of quadratic control dynamical systems with inherent stability guarantees. To this aim, we investigate the stability characteristics of control systems with  energy-preserving nonlinearities, thereby identifying conditions under which such systems are bounded-input bounded-state stable. These insights are subsequently applied to the learning process, yielding inferred models that are inherently stable by design. The efficacy of our proposed framework is demonstrated through a couple of numerical examples. 
	
}

\novelty{\begin{itemize}
		\item Learning stability-certified control systems with quadratic nonlinearities.
		\item Utilizing a stable matrix parameterization for the certification.
		\item Bounded-input bounded-state stability for quadratic systems with control is shown under parametrization assumptions.
		\item Several numerical examples demonstrate the stability-certified results of the learned models.
\end{itemize}}
\maketitle


\section{Introduction}
Significant developments have been made recently in the fields of learning dynamical systems from data, driven by the availability of large amount of data sets and the demand across various applications such as robotics, epidemiology, and climate science.  First principles model construction often falls short due to the complexity of these applications, paving the way for data-driven methodologies. Among these, the \emph{Dynamic Mode Decomposition} (DMD) together with Koopman operator theory is a common used approach, since it allows to handle complex dynamics by linearizing non-linear systems in a high-dimensional (or even infinite-dimensional) observable space, see \cite{Mez13, morSch10}. Additionally, the DMD framework was extended to allow handling control inputs in \cite{ProBK16}.  Recent advancements have also seen the emergence of \emph{Sparse Identification of Nonlinear Dynamics} (SINDy) for discovering the essential nonlinear terms from extensive libraries through sparse regression, offering an alternative perspective on nonlinear system identification (\cite{brunton2016discovering}). Moreover, the integration of prior knowledge, particularly physics-based, into learning frameworks has also attracted attentions, particularly through the operator inference (OpInf) technique (\cite{morPehW16}), which aim  at the construction  of data-driven (reduced-order) models by focusing on quadratic (or polynomial) dynamics.


Many physical phenomena demonstrate  stability behavior, i.e., their state variables evolve in a bounded region of the state space over extended time horizons. Consequently, the accurate representation of these phenomena necessitates stable differential equations, which is also essential for the long-term numerical integration. However, the critical aspect of stability is often not addressed while learning dynamical systems.  In this work, our main focus is to learn quadratic control  models of the form
\begin{equation}\label{eq:quad_model}
	\dot{\bx}(t) = \bA\bx(t) +  \bH\left(\bx(t)\otimes \bx(t)\right) + \bB\bu(t), \quad \bx(0) = \bx_0,
\end{equation}
where $\bx(t)\in\R^n$ is the state vector, $\bu(t)\in \R^m$ represents the inputs, and $\bA \in \R^{n \times n}$, $\bH \in \R^{n\times n^2}$, and $\bB \in \R^{n\times m}$ are the system matrices.  It is worth mentioning that quadratic models emerge inherently within discretized models of, e.g., fluid mechanics. Furthermore, through a process known as lifting transformation, smooth nonlinear dynamical systems can be transformed into quadratic dynamical systems \cite{morGu11}. Also, this transformation was leveraged to operator inference in \cite{QKPW2020_lift_and_learn} to learn physics-based quadratic models for a given data set, and identifying such a lifted transformation using neural networks was investigated in \cite{morGoyB22}.

 Our prior work concentrated on imposing stability on learned (uncontrolled) linear and quadratic models, see \cite{morGoyPB23, goyal2023guaranteed}. Therein,  we introduced methodologies to ensure local and global stability, leveraging parametrizations for stable matrices  and energy-preserving Hessians to enforce stability by design, mainly inspired by the results on energy-preserving nonlinearities proposed in \cite{schlegel2015long}. The approach \cite{morGoyPB23} not only addressed computational challenges but also bypassed the limitations of conventional stability constraints, thus providing a robust foundation for stable dynamical system modeling.

Building upon this foundation, our current work extends these concepts to quadratic systems with control inputs, enhancing their applicability in controlled environments and broadening the scope of data-driven dynamical system learning. To this aim, firstly in Section \ref{sec:OptInf}, we recapitulate the operator inference framework  for learning quadratic control  systems from (possibly) high dimensional data sets. Then, in Section \ref{sec:BoundQuadControlSystems},  we revisit the parameterizations of stable matrices and energy-preserving Hessians used in \cite{goyal2023guaranteed} to impose stability on learned uncontrolled quadratic models.  Then, we establish our main result, which consists of parametrizations of quadratic control systems that are guaranteed to be bounded-input bounded-state stable. Section \ref{sec:LearnDynStable} leverages the proposed parametrization to learn stable quadratic control systems. Subsequently, Section \ref{sec:ExtGenLyap} extends the presented results to Lyapunov functions of generalized quadratic form. Finally, Section \ref{sec:Num} illustrates the proposed methodology using a coupled of numerical examples, and Section \ref{sec:Conc} concludes the manuscript with summary and open avenues.


\section{Model Inference of quadratic control systems}\label{sec:OptInf}
We now present a brief overview of the Operator Inference (OpInf) methodology (\cite{morPehW16}), focusing specifically on systems with quadratic nonlinearities. 
Our discussion starts with learning quadratic models from high-dimensional data and considers dynamical system having a control input.

We proceed to define the problem of deriving low-dimensional dynamical systems from high-dimensional data originated from a nonlinear (psossibly quadratic) control system
\[ \dot\by(t) = \mathbf{f}(\by(t),\bu(t)), \quad t\geq 0, \] 
where $\by(t) \in \R^N$ is the state vector and $\bu(t)\in\R^m$ is an control input.
We assume the availability of the state snapshots $\by(t)$ at time $t \in \{t_0, t_1,\ldots,t_\cN\}$. These snapshots are aggregated into the snapshot matrix:
\begin{equation}\label{eq:HighDimData}
	\bY = \begin{bmatrix} \by(t_0),\ldots, \by(t_\cN) \end{bmatrix} \in \R^{N\times \cN}.
\end{equation}
Additionally, we assume that the dynamics is actuated by an input function $\bu(t)\in \R^m$ and that we have access to input snapshots at the same time steps, i.e., $\bu(t_0),\ldots, \bu(t_\cN)$ which can be aggregated into the matrix:
\begin{equation*}
\bU = \begin{bmatrix} \bu(t_0),\ldots, \bu(t_{\cN}) \end{bmatrix} \in \R^{m \times \cN}.
\end{equation*}
Although the dynamics of state $\by(t)$ are inherently $N$-dimensional, it can often be effectively represented in a low-dimensional subspace. As a result, we can aim at learning dynamics using the coordinate systems in the low-dimensional subspace. To that end, we identify a low-dimensional representation for $\by(t)$, which is done by determining the projection matrix $\bV \in \R^{N\times n}$, derived from the singular value decomposition of $\bY$ and selecting the $n$ most dominant  left singular vectors. This leads to the computation of the reduced state trajectory as follows:
\begin{equation}\label{eq:proj_step}
	\bX = \bV^\top \bY,
\end{equation}
where $\bX := \begin{bmatrix} \bx(t_0), \ldots, \bx(t_\cN)\end{bmatrix}$ with $\bx(t_i)= \bV^\top \by(t_i).$
With our quadratic model hypothesis, we then aim to learn the system operators from the available data. Precisely, our goal is to learn a quadratic control system of the form~\eqref{eq:quad_model},
where $\bA$, $\bH$, and $\bB$ are the system operators. 

Next, we cast the inference problem, which is as follows. Having the low-dimensional trajectories $\{\bx(t_0), \ldots, \bx(t_\cN)\}$ and the input snapshots $\{\bu(t_0), \ldots, \bu(t_\cN)\}$, we aim to learn operators $\bA$, $\bH$, and $\bB$ in \eqref{eq:quad_model}. At the moment, let us also assume to have (estimated) the derivative information of $\bx$ at time $\{t_0,\ldots,t_\cN\}$, which is denoted by $\dot{\bx}(t_0),\ldots, \dot{\bx}(t_\cN)$. Using this derivative information, we form the following matrix:
\begin{equation}
	\dot{\bX} = \begin{bmatrix} 		\dot{\bx}(t_0), \ldots,\dot{\bx}(t_\cN)	\end{bmatrix}.
\end{equation}
Then, 
determining the operators boils down to solving a least-squares problem, which can be written as
\begin{equation}\label{eq:opinf_optimization}
	\min_{\bA,\bH,\bB} \left\|\dot{\bX} - \begin{bmatrix}\bA,~\bH, ~ \bB\end{bmatrix} \cD \right\|_F,
\end{equation}
with $\cD = {\scriptsize \begin{bmatrix} \bX \\ \bX\totimes \bX \\ \bU \end{bmatrix} }$, where the product $\totimes$ is defined as	$\bG\totimes \bG = \begin{bmatrix}		\bg_1\otimes \bg_1,\ldots, \bg_\cN\otimes \bg_\cN 	\end{bmatrix}$ with $\bg_i$ being the $i$-th column of the matrix $\bG\in \R^{n\times \cN}$, and $\otimes$ denotes the Kronecker product. 
%
Additionally, the reader should notice that whenever the provided data in \eqref{eq:HighDimData} is already low-dimensional and further compression is not possible, then  the projection onto the POD coordinates in equation \eqref{eq:proj_step} is not required. 

Although the optimization problem~\eqref{eq:opinf_optimization} appears straightforward, it poses a couple of major challenges. 
Firstly,  the matrix $\cD$ can be ill-conditioned, thus making the optimization problem \eqref{eq:opinf_optimization} challenging. A way to circumvent this problem is to make use of suitable regularization schemes, and many proposals are made in this direction in the literature, see, e.g., \cite{morYilGBetal20,mcquarrie2020data,morBenGHetal20}.  

An important challenge---one of the most crucial ones---is related to  the stability of the inferred models. When the optimization problem \eqref{eq:opinf_optimization} is solved, then the inferred operators only aim to minimize the specific design objective. However, solving \eqref{eq:opinf_optimization} does not guarantee that the resulting dynamical system will be stable; The problem of imposing boundedness for quadratic system without control was studied in \cite{kaptanoglu2021promoting} using soft constraints. More recently, the authors in \cite{goyal2023guaranteed} proposed a parametrization for the system operators allowing to impose different types of stability properties on the learned uncontrolled dynamical models, such as local stability, global stability and the existence of trapping regions. In this paper, we build upon the concepts established in \cite{goyal2023guaranteed} for uncontrolled quadratic systems and extend those results to quadratic control system of the form \eqref{eq:quad_model}.

\section{Stability  for quadratic control systems}\label{sec:BoundQuadControlSystems}
In this section, our main goal is to parametrize quadratic control systems that are stable. To this aim,  we start by short reviewing the parametrization of stable matrices proposed in \cite{gillis2017computing}). Thereafter, we discuss the results presented in \cite{goyal2023guaranteed} allowing to parametrize energy preserving quadratic nonlinearities. Based on these, we subsequently characterize boundness for quadratic control systems with energy-preserving nonlinearities. These results will then be used to impose stability while learning of quadratic control systems.
\subsection{Parametrization of a stable matrix}
A linear dynamical system 
\begin{equation}\label{eq:LinearDyn}	\dot{\bx}(t) = \bA\bx(t),
	\end{equation}
where $\bA \in \R^{n \times n}$ is said to be asymptotically stable when all eigenvalues of the matrix $\bA$ lie strictly in the left half complex plane. In this case, the matrix $\bA$ is called Hurwitz.

In the light of this, an important characterization of stable matrices is provided in \cite[Lemma~1]{gillis2017computing}, which states that any Hurwitz matrix $\bA$  can be expressed as:
\begin{equation}\label{eq:stable_matrix_decomposition}
	\bA = (\bJ - \bR)\bQ,
\end{equation}
where $\bJ = -\bJ^{\top}$ is a skew-symmetric matrix,  $\bR = \bR^{\top} \succ 0$, and $\bQ = \bQ^\top \succ 0$ are symmetric positive definite matrices. Moreover, if $\bA$ can be expressed as in \eqref{eq:stable_matrix_decomposition}, then
$\bV(\bx) = \frac{1}{2}\bx^{\top}\bQ\bx $
is a Lyapunov function for the linear dynamical system \eqref{eq:LinearDyn}. In particular, the system \eqref{eq:LinearDyn} has $\bE(\bx) = \frac{1}{2} \bx^{\top}\bx$ as a (strict) Lyapunov function if and only if the matrix $\bA$ can be decomposed as
\begin{equation}\label{eq:ParMonStable}
 \bA = \bJ - \bR, 
\end{equation}
for $\bJ = -\bJ^{\top}$  and $\bR = \bR^{\top} \succ 0$. It this case, we say that the system \eqref{eq:LinearDyn} is \emph{monotonically stable}, because the 2-norm of the state, i.e., $\|\bx(t)\|_2$, always decreases with time. 

It is worth mentioning that the authors in  \cite{morGoyPB23} proposed a framework to learn linear stable dynamical system by levering the parametrization \eqref{eq:stable_matrix_decomposition}. Additionally, in \cite{goyal2023guaranteed}, this parametrization of $\bA$ together with the notion of energy preserving quadratic nonlinearities plays a crucial role in learning (uncontrolled) quadratic models from data. In what follows, we recall the notion of energy preserving quadratic nonlinearities and their parametrization.

\subsection{Energy preserving quadratic nonlinearity}
Here, we examine quadratic dynamical systems as described in \eqref{eq:quad_model} for which the quadratic nonlinearity satisfies some algebraic constraints, as proposed e.g., in~\cite{lorenz1963deterministic,schlegel2015long}. We refer the Hessian matrix, or quadratic matrix $\bH$ in \eqref{eq:quad_model} to as energy-preserving when it satisfies the following criteria:
\begin{equation}\label{eq:EnergPreservingCondition}
	\bH_{ijk} + \bH_{ikj} + \bH_{jik} + \bH_{jki} + \bH_{kij} + \bH_{kji} = 0, 
\end{equation}
for each $ i,j,k \in \{1, \dots, n\}$, with $\bH_{ijk} := e_i^{\top} \bH(e_j \otimes e_k)$, where $e_i \in \R^n$ is the $i$th canonical basis vector. The condition \eqref{eq:EnergPreservingCondition} can be expressed using  Kronecker product notation, see \cite{goyal2023guaranteed}, as follows:
\begin{equation}\label{eq:EnergPreservingConditionKron}
	\bz^{\top}\bH(\bz\otimes \bz) = 0, \quad \forall \bz \in \R^n. 
\end{equation}
Energy preserving nonlinearities  typically appear in finite element discretization of fluid mechanical models with certain boundary conditions, see, e.g.,  \cite{holmes2012turbulence, schlichting2016boundary}, 
as well as, in magneto-hydrodynamics applications, see \cite{kaptanoglu2021structure}. 
For such uncontrolled quadratic systems (the system \eqref{eq:quad_model} with $\bu \equiv 0$) with energy-preserving Hessian, it is possible to establish conditions that ensure the system's energy, defined by $\bE(\bx(t)) := \frac{1}{2} \bx^{\top}(t)\bx(t) = \frac{1}{2}\|\bx(t)\|_2^2$, decreases in a strictly monotonic fashion for all trajectories, see \cite{schlegel2015long,goyal2023guaranteed}. Additionally, the authors in \cite{goyal2023guaranteed} (Lemma 2) showed that a Hessian matrix $\bH \in \R^{n \times n^2}$ satisfying \eqref{eq:EnergPreservingConditionKron} can be parametrized without loss of generality as 
\begin{equation}\label{eq:str_H}
	\bH = \begin{bmatrix} \bH_1 & \ldots& \bH_n \end{bmatrix},
\end{equation}
with $\bH_i \in \Rnn$ being skew-symmetric, i.e., $\bH_i = -\bH_i^\top$. As a consequence, this parametrization is leveraged in the learning process in \cite{goyal2023guaranteed}, leading to the inference of stable quadratic (uncontrolled) models. Furthermore, the authors in \cite{goyal2023guaranteed} generalized these results to the case of more general quadratic Lyapunov functions.

\subsection{Bounded-input bounded-state stability result}
Based on the results presented so far, we now proceed further  to establish our main results of this paper on the stability for quadratic control systems. 

\begin{theorem}\label{theo:BoundEq} Consider a quadratic control  system as in \eqref{eq:quad_model}.  Assume that the matrix $\bA\in \R^{n \times n}$ is monotonically stable and can be decomposed as $\bA = \bJ-\bR$, where $\bJ = -\bJ^{\top}$  and $\bR = \bR^{\top} \succ 0$, and $\bH \in \R^{n\times n^2}$ is an energy-preserving Hessian. Then, if the input function $\bu \in L_{\infty}$, then the state vector $\bx(t)$ monotonically converges to the interior of the ball $\cB_{r}(0)$, where 
	\[ \displaystyle r = \dfrac{\|\bB\|_2\|\bu\|_{L_{\infty}}}{\sigma_{\min}(\bR)}, \]
$\sigma_{\min}(\cdot)$ is the minimum singular value of a matrix and $\|\bu\|_{L_{\infty}} = \texttt{ess sup}_{t\geq 0}\|\bu(t)\|_2$ . Furthermore, for every input $\bu \in L_{\infty}$, the state vector is bounded 
\begin{equation}\label{eq:bound_state}
 \|\bx(t)\|_2 \leq \max\{\bx_0, r\} ;
\end{equation}
thus, the quadratic control system is bounded-input bounded-state stable.
\end{theorem}
\begin{proof}
Let us consider the energy function $\bE(\bx) = \frac{1}{2}\bx^{\top}\bx$. We then show that this function is a Lyapunov function outside of $\cB_{r}(0)$, which proves the result.  We utilize the parameterization of the matrix $\bA = \bJ-\bR$, where $\bJ = -\bJ^{\top}$  and $\bR = \bR^{\top} \succ 0$ and the matrix $\bH$ to be energy-preserving . To this aim, the derivative of $\bE(\bx(t))$ along the trajectory $\bx(t)$ is given by
\begin{align*}
	\dot{\bE}(\bx(t)) &= \bx(t)^\top\dot\bx(t)\\ 
	&= \bx(t)^\top \left(\bA \bx(t) + \bH\left(\bx(t)\otimes \bx(t)\right) +\bB\bu(t)\right)\\
	& = \bx(t)^\top \left(\left(\bJ - \bR\right) \bx(t) + \bH\left(\bx(t)\otimes \bx(t)\right) +\bB\bu(t)\right)\\
	& = - \bx(t)^\top\bR \bx(t) +  \cancelto{0}{\bx(t)^\top\bH\left(\bx(t)\otimes \bx(t)\right)} \\ &\qquad +\bx(t)^{\top}\bB\bu(t)\\
	& \leq -\sigma_{\min}(\bR)\|\bx(t)\|^2_2 + \|\bB\|_2\|\bu\|_{L_{\infty}}\|\bx(t)\|_2.
\end{align*}
Define $ r = \dfrac{\|\bB\|_2\|\bu\|_{L_{\infty}}}{\sigma_{\min}(\bR)}$. Then, $\dot{\bE}(\bx(t))  < 0$ for $\|\bx(t)\|_2 >r$. Hence, $\bE(\bx(t))$ is a Lyapunov function for $\|\bx(t)\|> r$, and hence, it is monotonically decreasing outside of the ball $\cB_r(0)$. As a consequence, the state norm $\|\bx(t)\|_2$ is also a monotonically decreasing function outside of $\cB_r(0)$. This implies that $\|\bx_0\|_2 \geq \|\bx(t)\|$ when $\|\bx_0\|_2 \geq r$, which proves the result in \eqref{eq:bound_state}. As a consequence, every input $\bu \in L_{\infty}$ will lead to bounded trajectories and the system is bounded-input bounded-state stable. 
\end{proof}

Theorem \ref{theo:BoundEq} shows that if $\bA$ is monotonically stable and $\bH$ is energy preserving, the quadratic control system of the form \eqref{eq:quad_model} is bounded-input bounded-state stable, i.e., every input $\bu \in L_{\infty}$ will lead to bounded trajectories $\bx(t)$ satisfying \eqref{eq:bound_state}. Additionally, to prove this result we use the state energy as a Lyapunov function. 

 Theorem \ref{theo:BoundEq} together with the parametrization of monotonically stable matrices in \eqref{eq:ParMonStable} and energy-preserving Hessians in \eqref{eq:str_H} will allow us to learn quadratic control systems with a stable behavior. 

\section{Learning bounded control systems}\label{sec:LearnDynStable}

Based on Theorem \ref{theo:BoundEq}, we establish an inference framework to learn bounded quadratic control  models of the form \eqref{eq:quad_model}, via the designated $\bX$, $\bU$ and $\dot{\bX}$ dataset. The inference problem is formulated as:
\begin{equation}\label{eq:stable_learning}
	\begin{aligned}
 \underset{\hat\bJ, \hat\bR, \hat\bH_1,\ldots,\hat\bH_n, \hat\bB}{\arg\min}  \left\|\dot{\bX}-\hat\bA\bX - \hat \bH\bX^\otimes -\hat\bB\bU \right\|_F,
 	\\ \qquad \text{where} ~~ \hat\bA = (\hat\bJ - \hat\bR)  \quad \text{and} \quad \hat\bH = \begin{bmatrix}\hat\bH_1,\ldots,\hat\bH_n\end{bmatrix},
		\\ \qquad \text{subject to} ~~ \hat\bJ = -\hat\bJ^\top,  \hat\bR =  \hat\bR^\top \succ 0,~\text{and}\\
		\hspace{2.5cm}\hat{\bH}_i = -\hat{\bH}^{\top}_i,~i\in\{1,\ldots,n\}.&
	\end{aligned}
\end{equation}
Upon determining the optimal set $(\bJ, \bR, \bH_1, \ldots, \bH_n, \bB)$ solving \eqref{eq:stable_learning}, the matrices $\bA$ and $\bH$ are constructed as:
\begin{equation}
	\bA = (\bJ - \bR), \quad \bH = \left[\bH_1, \ldots, \bH_n \right],
\end{equation}
yielding a quadratic control  system in the form of \eqref{eq:quad_model}, which is guaranteed to be stable in the view of Theorem~\ref{theo:BoundEq}. Notice that the problem formulation in \eqref{eq:stable_learning} imposes certain constraints on the matrices. To circumvent these restrictions, as used in \cite{morGoyPB23, goyal2023guaranteed}, skew-symmetric matrices $\hat\bJ$ (or $\hat\bH_k$) and symmetric positive (semi)definite matrices $\tilde\bR$  can be parameterized as
\begin{equation}\label{eq:par_ss}
	\hat\bJ = \bar\bJ - \bar\bJ^\top, \quad \text{and}\quad 	\tilde\bR = {\bar\bR}\bar\bR^\top, 
\end{equation}
where $\bar\bJ, {\bar\bR}\in \R^{n\times n}$ are square matrices with any constraints. With this parametrization, \eqref{eq:stable_learning} becomes an unconstrained optimization problem. However, due to the lack of analytical solution to \eqref{eq:stable_learning}, we solve the problem using a gradient-based approach. 


\section{Extension to more general quadratic Lyapunov functions}\label{sec:ExtGenLyap}

Theorem \ref{theo:BoundEq} shows that a quadratic control system with $\bA$ monotonically stable and $\bH$ energy preserving is guaranteed to be bounded-input bounded-state stable. We have proved the result using the state energy $\bE(\bx) = \frac{1}{2}\bx^{\top}\bx$ as the Lyapunov function. In this section, we sketch an extension of this result to the case for which more general quadratic Lyapunov functions are considered, i.e.,  functions of the form $\bV(\bx) =  \bx^{\top}\bQ\bx$, where  $\bQ = \bQ^\top \succ 0$ is a symmetric positive definite matrix. 

Let us assume that the matrices $\bA$ and $\bH$ of  a quadratic control system of the form \eqref{eq:quad_model} can be written as 
\begin{align}\label{eq:GenPar}
\bA =  (\bJ - \bR)\bQ\quad \text{and} \quad \bH = \begin{bmatrix} \bH_1\bQ & \ldots& \bH_n\bQ \end{bmatrix}.
\end{align}
In this case, $\bA$ is a Hurwitz matrix and $\bH$ is a generalized energy-preserving Hessian (see \cite{goyal2023guaranteed}). With similar arguments as those in the proof of Theorem \ref{theo:BoundEq}, one can show that for bounded inputs $\bu \in L_{\infty}$, the state $\bx(t)$ also has a bounded behavior. Indeed, to this aim, one needs to use $\bV(\bx)$ as a Lyapunov function, and the result follows straightforwardly. As a consequence, the parametrization in \eqref{eq:GenPar} can be leveraged within the learning process, i.e., the optimization problem \eqref{eq:stable_learning} can incorporate this more general parametrization, thus yielding inferred quadratic control systems to be stable by construction. 

%
\section{Numerical results}\label{sec:Num}
In this section, we assess the efficacy of the methodology outlined in \eqref{eq:stable_learning}, referred to herein as \newmethod, through a coupled of numerical examples.  
We compare our approach with operator inference \cite{morPehW16}, which we denote as (\benchmark).
All experiments are carried out using \texttt{PyTorch}, with $12,000$ updates with the Adam optimizer (\cite{kingma2014adam}) and a triangular cyclic learning rate ranging from $10^{-6}$ to $10^{-2}$. Additionally, we regularize the matrix $\bH$ in quadratic systems by adding $10^{-4}\cdot \|\bH\|_{l_1}$ in the loss function, where $\|\cdot \|_{l_1}$ denotes $l_1$-norm. The initial values for the matrix coefficients are randomly generated from a Gaussian distribution with a mean of $0$ and standard deviation of $0.1$. 

\subsection{Low-dimensional example I}
Our first numerical example consists of a low-dimensional quadratic control system of the form \eqref{eq:quad_model}, where
\begin{equation}\label{eq:quad_2d_model}
	\bA = \spalignmat[r]{-1 1; -1 -2},  ~~ \bH =\spalignmat[r]{0 1 0 0; -1 0  0 0}, ~~\text{and} ~~ \bB =\spalignmat[r]{1;1}. 
\end{equation}
We collect the data with zero initial condition and two different training input functions of the form 
\begin{equation}\label{eq:2d_train_inputs}
	\bu(t) = \sin(f_1 t)e^{-f_2t} + \sin(g_1 t)e^{-g_2t},
\end{equation}
where $f_i \in \mathbb{Z}$, $i = \{1,2\}$ are randomly chosen integers between $0$ and $5$, and $g_i\in \R$, $i = \{1,2\}$ are randomly chosen real numbers between $0$ and $0.5$. 
We collect $200$ points for each training input in the time-span of $[0,10]$. 
Then, we learn quadratic control models using \newmethod~and \benchmark. Since the data is low-dimensional, the proper orthogonal decomposition step in \eqref{eq:proj_step} is not  performed in this example.


For comparison, we consider  two test control inputs of the form 
\begin{subequations}
	\begin{align*}
		\bu_1(t) &= \sin(t)e^{(-0.2 \cdot t)} + \sin(2t)e^{(-0.6\cdot t)} 
		 + \cos(3 t)e^{(-t)}\\
		\bu_2(t) &= -\sin(2t)e^{(-0.1 \cdot t)} - \sin(t)e^{(-0.3\cdot t)}
		 + \cos(4t)e^{(-0.5t)}.
	\end{align*}
\end{subequations}
Note that the testing inputs are very different than from the training ones (see \eqref{eq:2d_train_inputs}). Next, we compare the time-domain simulations of the learned models with the ground truth and the results are depicted in \Cref{fig:lowdim_example}. We notice a faithful learning of the underlying models using both approaches; however, the proposed methods \newmethod\ ensures stability by construction for any other selected input.

  \begin{figure*}[tb]
	\begin{center}
		\begin{subfigure}[b]{0.8\textwidth}
			\includegraphics[width = 1\textwidth]{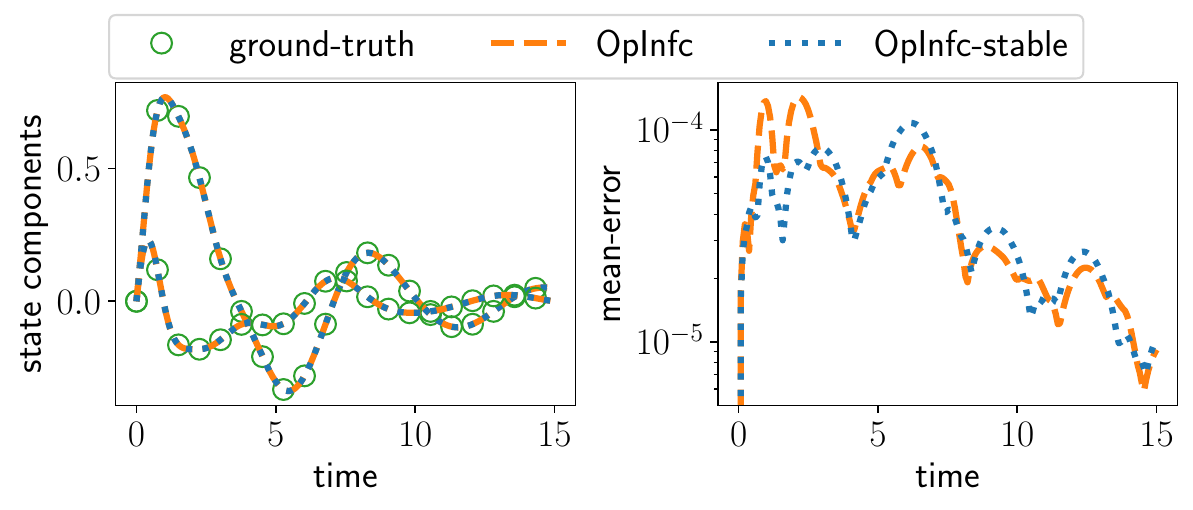}
			\caption{For the testing input $\bu_1$.}
		\end{subfigure}
		\begin{subfigure}[b]{0.8\textwidth}
			\includegraphics[width = 1\textwidth]{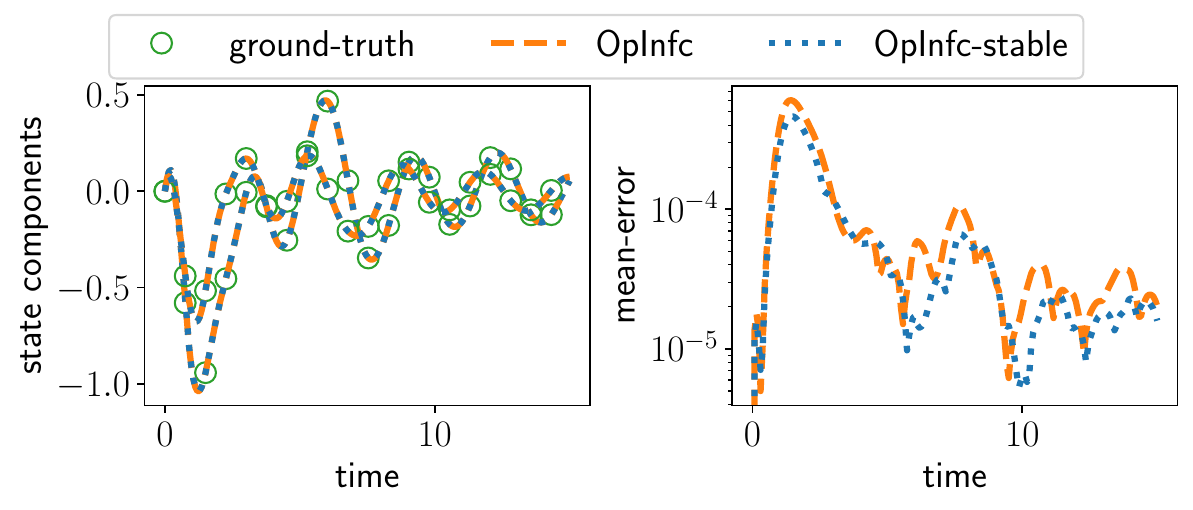}
			\caption{For the testing input $\bu_2$.}
		\end{subfigure}
		\caption{Low-dimensional example I:  A performance test for testing control inputs of the inferred models.}
		\label{fig:lowdim_example}
	\end{center}
\end{figure*}

\subsection{Low-dimensional example II}
Our second numerical example, we consider a slightly different  quadratic control system than the previous example whose the matrix $\bA$ is as follows:
\[ \bA = 0.01\spalignmat[r]{-1 1; -1 -2}, \]
and the matrices $\bH$ and $\bB$ are the same as \eqref{eq:quad_2d_model}. We collect training data using zero initial conditions and two controlled inputs with the same setting as the previous example. 
%
Moreover, for this example, we add Gaussian noise of zero mean and $0.02$ standard derivation in the training data. Then, we learn quadratic controlled inputs using \benchmark~and \newmethod. We compare the qualities of these two models using testing control inputs, similar to the previous example. Particularly, to test the stability of both models, we use high-magnitude test control inputs as follows:

\begin{subequations}
	\begin{align*}
		\bw_1(t) &= 10\cdot \Big(\sin(t)e^{(-0.2 \cdot t)} + \sin(2t)e^{(-0.6\cdot t)} 
		+ \cos(3 t)e^{(-t)}\Big)\\
		\bw_2(t) &= 10\cdot\Big(-\sin(2t)e^{(-0.1 \cdot t)} - \sin(t)e^{(-0.3\cdot t)} 
			+ \cos(4t)e^{(-0.5t)}\Big).
	\end{align*}
\end{subequations}

The time-domain simulations using the learned models for these two test control inputs are shown in \Cref{fig:lowdim_example_2}. We notice that \benchmark~yields unstable behaviors, particularly for the control $\bw_2$, whereas \newmethod~results into the models which are stable by construction and this phenomena is also numerically observed. 

\begin{figure*}[tb]
	\begin{center}
		\begin{subfigure}[b]{0.8\textwidth}
			\includegraphics[width = 1\textwidth]{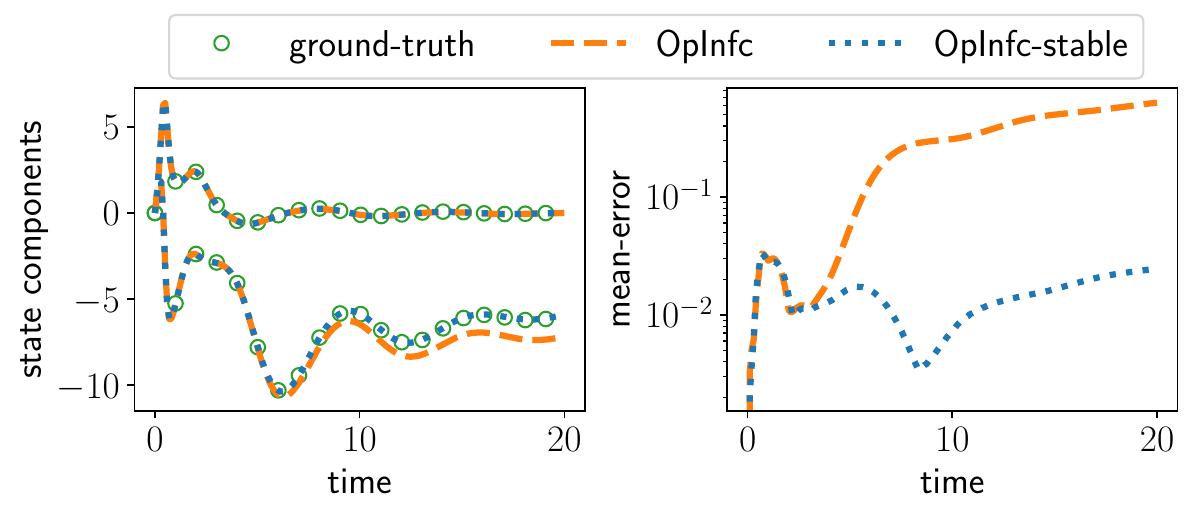}
			\caption{For the testing input $\bw_1$.}
		\end{subfigure}
		\begin{subfigure}[b]{0.8\textwidth}
			\includegraphics[width = 1\textwidth]{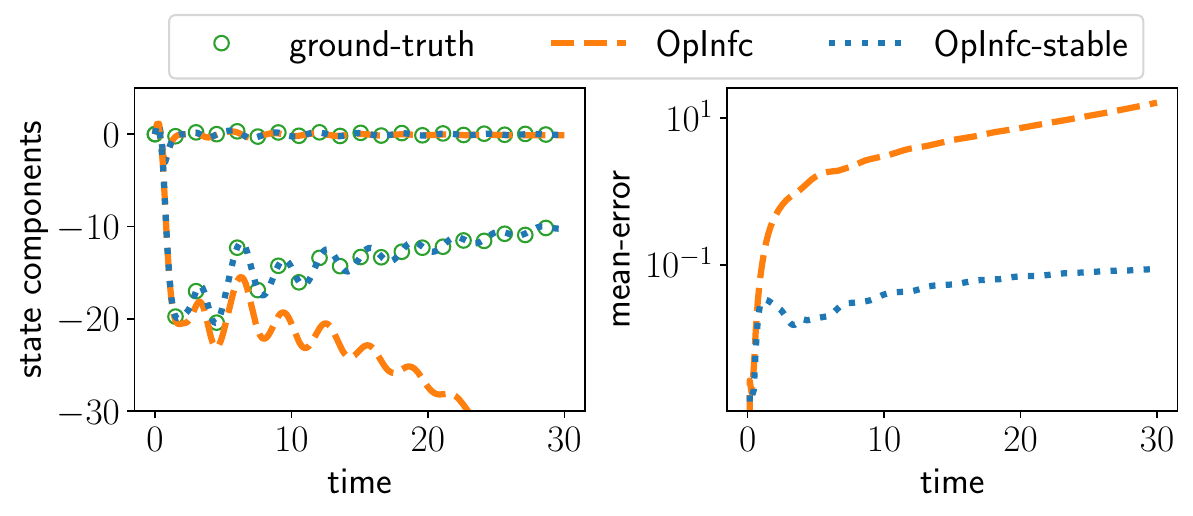}
			\caption{For the testing input $\bw_2$.}
		\end{subfigure}
		\caption{Low-dimensional example II:  A performance test for testing control inputs of the inferred models.}
		\label{fig:lowdim_example_2}
	\end{center}
\end{figure*}

\subsection{High-dimensional Burgers' example}
In our next example, we consider viscus Burgers' example, whose  governing equation is as follows:
\begin{equation}
	\begin{aligned}
		&\dfrac{\partial v}{\partial t} + v\dfrac{\partial v}{\partial \xi} = \mu \dfrac{\partial^2}{\partial \xi^2} + f(\xi,t),\\
		&v(0,t) = 0, ~~v(L,t) = 0,\\
		&v(\xi,0) = 0,
	\end{aligned}
\end{equation}
where $\xi\in [0,L]$ and $t$ denote space and time, respectively, and $v(\xi,t)$ denotes the state variable at the spatial location $\xi$ and at time $t$. We set $\mu = 0.05$ and $L = 2$. Moreover, $f(\xi,t)$ denotes a source term, and in this example, we assume that the source term $f(\xi,t)$ is separable, i.e., $b(\xi)u(t)$. Additionally, we consider \[b(\xi) = \cos\left(\left(\dfrac{\xi}{L}-1\right)\dfrac{\pi}{2}\right).\]
Note that the consider Burgers' example has Dirichlet boundary conditions at both boundary ends. Hence, the quadratic term is energy-preserving. 

We discretize the governing equation using a finite difference scheme by considering $251$ points in the space. For generating training data, we consider the control input $\bu(t)$ of the form
\begin{equation}\label{eq:burgers_train_inputs}
	\bu(t) = \sin(f_1 t)e^{-g_1t} + \sin(f_2 t)e^{-g_2t},
\end{equation}
where $f_1$ and $f_2$ are randomly drawn from a Gaussian distribution of $\cN(0,2)$ and $g_1$ and $g_2$ are randomly drawn from a uniform distribution of $\cU(0.1, 1.1)$. We consider $20$ different training inputs and for each train input, we take $1001$ points at equidistant in the time $[0,10]$. 

Towards learning quadratic control models, we first aim at determining a suitable low-dimensional representation of the high dimensional data. It is done by means of singular value deposition of the training data. We project the high-dimensional data onto a lower dimensional subspace using the most dominant left singular vectors. We take nine most dominant ones which captures more than $99.90\%$ energy present in the training data. Furthermore, to learn quadratic models, we require derivative information, which is estimate using 5-order stencils. 

Next, we learn quadratic models using \benchmark~and \newmethod~using the projected low-dimensional data. To capture the qualities of both these learned models, we consider testing control inputs, which takes the following form:
\begin{equation}\label{eq:burgers_test_inputs}
	\bu(t) = \sin(f_1 t)e^{-g_1t} + \sin(f_2 t)e^{-g_2t} + \cos(f_3 t)e^{-g_3t},
\end{equation}
where $f_i \in \cN(0,2)$ and $g_i \in \cU(0.1, 1.1)$, $i \in \{1,2,3\}$. We run testing for $10$ different testing control inputs which are quite different than the training ones. Among $10$ test cases, for one case, we present time-domain solutions obtained using the learned models and compare with the ground truth in \Cref{fig:burgersexample_timedomain}. For a comparison for all $10$ test cases, we compute the following measure:
\begin{equation}\label{eq:err_measure}
	\texttt{err} = \texttt{mean}\left(\bX^{\texttt{ground-truth}} - \bX^{\texttt{learned}}\right),
\end{equation}
where $\bX^{\texttt{ground-truth}}$ and $\bX^{\texttt{learned}}$ contain the solutions vector at all time $t$ for the ground truth and learned quadratic models, respectively. Based on the measure \eqref{eq:err_measure}, we compute the errors based on the solutions obtained using \benchmark~and \newmethod~and plot in \Cref{fig:burgersexample_errorplot}. We notice that a slightly better performance for \newmethod~despite enforcing stability parameterization.

  \begin{figure}[t!]
		\begin{center}
		\begin{subfigure}[b]{0.8\textwidth}
			\includegraphics[height=4.8cm]{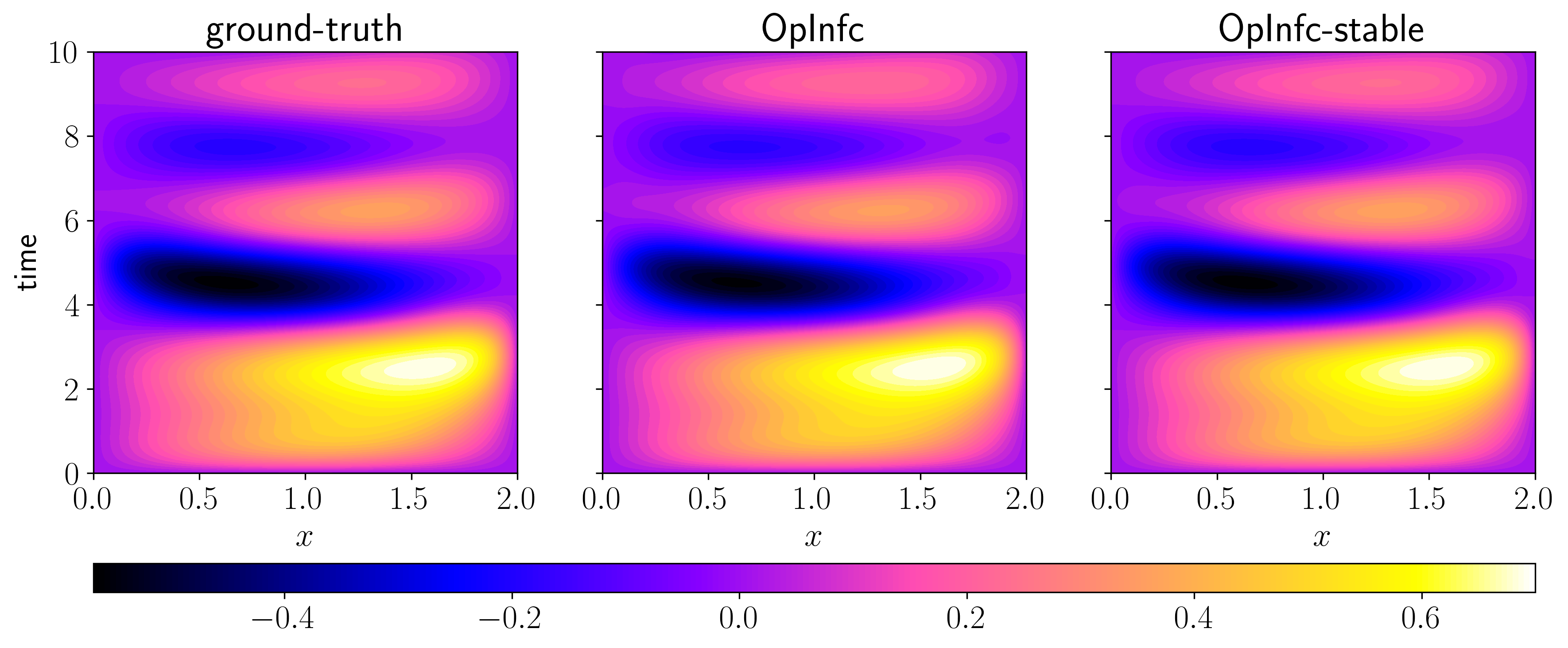}
			\caption{Time-domain response.}
		\end{subfigure}
		\begin{subfigure}[b]{0.55\textwidth}
			\includegraphics[height=4.8cm]{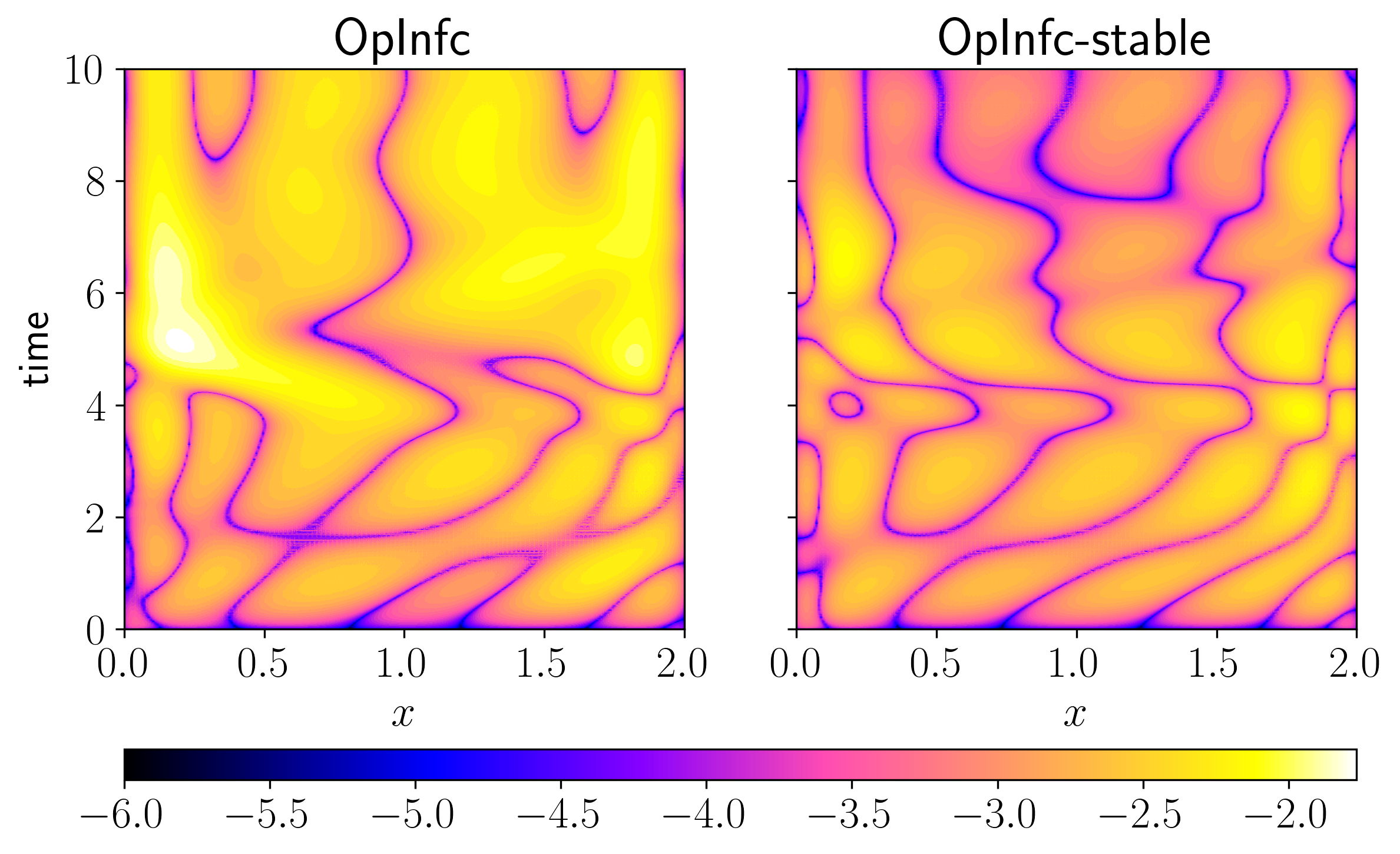}
			\caption{Absolute error in a log-scale.}
		\end{subfigure}
		\caption{Burgers' example:  A performance test for a testing control input of the inferred models.}
		\label{fig:burgersexample_timedomain}
	\end{center}
\end{figure}

\begin{figure}[t!]
	\begin{center}
		\includegraphics[height=4.8cm]{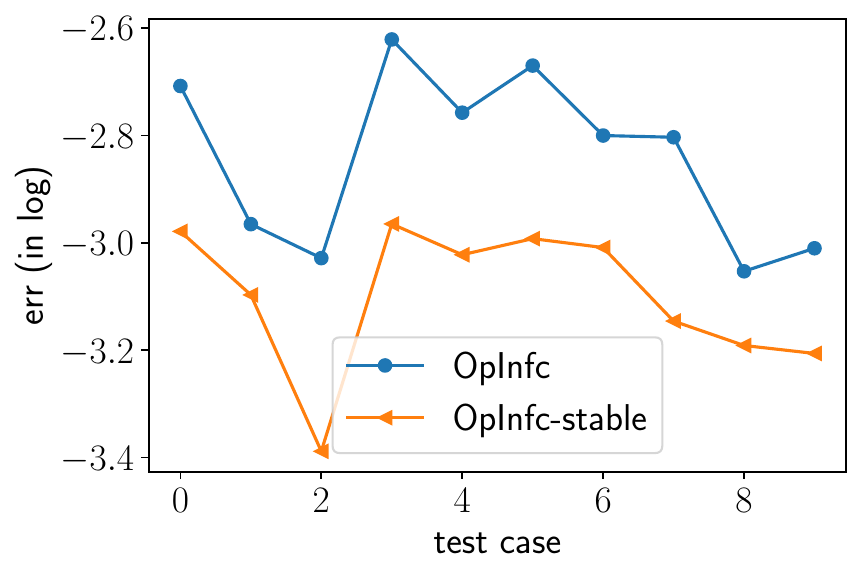}
		\caption{Burgers' example: A comparison of \benchmark~and \newmethod~for $10$ test cases.}
		\label{fig:burgersexample_errorplot}
	\end{center}
\end{figure}

\section{Conclusions}\label{sec:Conc}

In this paper, we introduced a data-driven methodology designed to ensure a bounded stability of the learned quadratic control systems. Firstly, under the assumption that the linear operator is stable and the quadratic operator is energy-preserving, we have showed that quadratic control systems is bounded-input bounded-state stable. Leveraging our previous work \cite{morGoyPB23}, we have parameterized the matrices of a quadratic system, satisfying stability and energy-preserving hypotheses by construction. And we have utilized the matrix parameterizations in a data-driven setting to obtain stable quadratic control systems.
%
We have discussed the effectiveness of our proposed methodology using two numerical examples and have compared the results when stability is not enforced. The results highlight the robust performance ad stability-certificates of the proposed approach, affirming its potential to significantly advance the field of data-driven learning of dynamical systems. 

In our methodology, we require accurate derivative information, which can be difficult to estimate if data are noisy and sparse. To avoid this requirement, we can incorporate integrating scheme  or the concept of neural ODEs \cite{chen2018neural}. With this spirit, methodologies to learn uncontrolled dynamical systems are discussed, e.g., in \cite{goyal2022discovery,uy2022operator} which will be adaopted to controlled cases in our future work.

\bibliographystyle{elsarticle-num} 
\bibliography{mor,igorBiblio}

\end{document}